\documentclass[journal]{IEEEtran}

\usepackage{bm}

\usepackage{amsthm}
\usepackage{amsfonts}
\usepackage{mathtools}
\usepackage{color}
\usepackage[font=small,labelfont=bf]{caption}
\usepackage[ruled,vlined]{algorithm2e}

\usepackage{caption}
\usepackage{subcaption}

\makeatletter
\newcommand*{\rom}[1]{\expandafter\@slowromancap\romannumeral #1@}
\makeatother

\renewcommand{\(}{\left(}
\renewcommand{\)}{\right)}
\renewcommand{\[}{\left[}
\renewcommand{\]}{\right]}

\renewcommand{\a}{\mathbf{a}}
\renewcommand{\b}{\mathbf{b}}

\newcommand{\thet}{\bm{\theta}}

\newcommand{\y}{\mathbf{y}}
\renewcommand{\H }{\mathbf{H}}
\newcommand{\z}{\mathbf{z}}
\newcommand{\w}{\mathbf{w}}
\newcommand{\p}{\mathbf{p}}

\newcommand{\x}{\mathbf{x}}

\renewcommand{\P}{\mathbf{P}}

\newcommand{\A}{\mathbf{A}}

\newcommand{\h}{\mathbf{h}}
\renewcommand{\u}{\mathbf{u}}

\newcommand{\Q}{\mathbf{Q}}

\newcommand{\q}{\mathbf{q}}

\renewcommand{\b}{\mathbf{b}}

\renewcommand{\log}[1]{{\rm{log}}#1}

\newtheorem{assumption}{Assumption}
\newtheorem{lemma}{Lemma}

\usepackage{bbm}
\usepackage[hyphens]{url}
\usepackage{graphicx} 
\graphicspath{{images/}}

\begin{document}

\title{Convex Nonparanormal Regression}

\author{Yonatan Woodbridge, Gal Elidan and Ami Wiesel
\thanks{The authors are with the Hebrew University of Jerusalem, Israel, and Google Research. This research was partially supported by the Israel Science Foundation (ISF) under Grant 1339/15}}

\maketitle

\begin{abstract}
    Quantifying uncertainty in predictions or, more generally, estimating the posterior conditional distribution, is a core challenge in machine learning and statistics. We introduce Convex Nonparanormal Regression (CNR), a conditional nonparanormal approach for coping with this task. CNR involves a convex optimization of a posterior defined via a rich dictionary of pre-defined non linear transformations on Gaussians. It can fit an arbitrary conditional distribution, including multimodal and non-symmetric posteriors. For the special but powerful case of a piecewise linear dictionary, we provide a closed form of the posterior mean which can be used for point-wise predictions. Finally, we demonstrate the advantages of CNR over classical competitors using synthetic and real world data. 
\end{abstract}


\section{Introduction}
A fundamental task in data analysis is fitting a conditional
distribution to an unknown label given observed features.
Standard methods provide a point estimate \cite{wasserman2006all,wasserman2013all}, yet the holy grail
is a full characterization of the posterior distribution. Linear
regression (LR) is the simplest case with a closed form
solution that is optimal when the underlying posterior is a
Gaussian distribution with a fixed variance. More complex
models involve challenging optimizations that fit expressive
distributions. The main contribution of this letter is the introduction
of a general class of non-symmetric, heavy tailed
or multimodal conditional distributions that can be efficiently
fitted using convex optimization methods.

There is a large body of work on derivation of the full
posterior of unknown labels \cite{koenker1978regression,meinshausen2006quantile,goldberg1997regression}. In this context, LR can be
interpreted as the Maximum Likelihood (ML) estimate of
a fixed-variance Gaussian distribution whose mean depends
linearly on the features. An extension to LR, which we denote
by Gaussian Regression (GR), also assumes that the variance
depends on the features (based on \cite{bishop1994mixture}), and its ML estimator involves a
standard convex optimization. Other works allow a variance
that depends non-linearly on the features \cite{nix1995learning,heskes1997practical,papadopoulos2001confidence,meinshausen2006quantile}.
However, these all implicitly assume a unimodal symmetric
posterior, and cannot capture complex multimodal posteriors.

On the other extreme, recent trends in machine learning
motivate the characterization of distributions via non-linear
transformations on known latent variables. This is the main
building block of the nonparanormal distribution (NPN) \cite{liu2009nonparanormal},
which is based on transformation functions to the Gaussian.
Similarly, Gaussian copulas allow arbitrary marginals with
a Gaussian dependence structure, and belong to a richer
family of copula models \cite{nelsen2007introduction,klaassen1997efficient,davy2003copulas,iyengar2011parametric,woodbridge2017signal}. These models
led to impressive results in modeling high dimensional non-
Gaussian data, but do not consider distributions which are
conditioned on observed features. More recently, normalizing flow networks were introduced \cite{tabak2013family,kobyzev2019normalizing,rezende2015variational,tabak2010density,dinh2016density}, and
involve a deep composition of invertible maps through network
layers. These provide a rich class of transformations, but their
optimization is usually non-convex and poorly understood.

In this letter, we introduce the CNR, a conditional 
version of the NPN. CNR relies on a linear combination
of pre-defined basis functions that provides a general class
of non-Gaussian, and possibly multimodal posteriors. With
a rich enough dictionary, this approach can model arbitrary
conditional distributions. The price comes at an increased
number of unknown parameters and higher sample complexity.
By design, the likelihood function of CNR is convex and can be efficiently optimized using state of the art optimization
techniques, e.g., Alternating Directions Method of Multipliers.

On the practical side, we focus on a specific CNR implementation
based on a dictionary of piecewise linear transformation
functions, allowing us to derive a closed form solution to
the conditional CNR expectation used for prediction. The
dictionary partitions the label’s domain into predefined bins,
and uses a different transformation in each bin. By choosing
the number of bins, this allows for a flexible tradeoff
between complexity and expressive power. We demonstrate
the performance advantages of our CNR implementation over
LR and GR, using numerical experiments in both synthetic
and real world data. Given enough samples, CNR results in
better likelihood values than its competitors, and succeeds to fit multimodal posteriors. Given enough samples, CNR prediction performs similarly to both LR and GR, which are sufficient for point-wise estimation. 

\textbf{Letter organization}. We begin in Section II by
introducing the general CNR framework.  
In Section
III we present our specific CNR implementation based on a
dictionary of piecewise linear functions. Finally, in Section
IV we present the results of our numerical experiments.

\section{Convex Nonparanormal Regression}
In this section, we introduce Convex Nonparanormal Regression (CNR), a non-Gaussian yet convex generalization of linear regression. CNR learns the posterior of $y\in\mathbb{R}$ given a $k$-vector $\x\in\mathbb{R}^k$ and is parameterized by ${\bm{\theta}}\in\mathbb{R}^d$,
\begin{equation*}
    p_{\bm{\theta}}(y|\x),
\end{equation*}
CNR tries to estimate the unknown ${\bm{\theta}}$ given pairs of $(y,\x)$. It parameterizes the posterior using an expressive class of transformation functions to the Gaussian distribution. For this purpose, we define
\begin{equation*}
    g_{\bm{\theta}}(y;\x)
    \sim\mathcal{N}(0,1).
\end{equation*}
The transformation function $g$ operates on $y$ and is determined by the arguments $\thet$ and $\x$.
We assume that the pair $(\x,\thet)$ is restricted to the set $(\mathcal{X},\Theta_{\x})$, where $\Theta_{\x}$ depends on $\x$.
We require the following assumptions on these sets:
\begin{assumption}
The transformation function $g_{\thet}(y;\x)$ is continuous, piecewise differentiable and monotonically increasing with respect to $y$ for any $\x\in{\mathcal{X}}$ and $\thet\in\Theta_{\x}$.
\end{assumption}
\begin{assumption}
The set $\Theta_{\x}$ is convex in $\thet$ for all $\x\in{\mathcal{X}}$. Within this set, the transformation $g$ is affine in $\thet$.
\end{assumption}
We require a monotonically increasing $g_{\thet}(y;\x)$ in order to ensure its invertibility.
We emphasize that these assumptions allow $g$ to be highly non-linear. As we shall see, by choosing an expressive class of $g$, the CNR can approximate arbitrary non-Gaussian posteriors.


\begin{lemma}[Convexity]
Suppose that Assumption 1 holds. For any point $y$ which is differentiable in $g_{\thet}(y;\x)$, the negative log likelihood parameterized by $\thet$
(ignoring all constants) is:
\begin{eqnarray*}
-\log\;p_{\bm{\theta}}(y|\x) = g_{\thet}^2(y;\x)-2\log\(
g'_{\bm{\theta}}(y;\x)
\)
\end{eqnarray*}
where the derivative is defined as
\begin{equation*}
    g'_{\bm{\theta}}(y;\x)
    = \frac{\partial g_{\bm{\theta}}(y;\x)}{\partial y}
\end{equation*}
Further, together with assumption 2, this objective is convex in $\thet\in \Theta_{\x}$ for all $\x\in{\mathcal{X}}$.
\end{lemma}
Note that the likelihood function is defined almost everywhere, since $g_{\thet}(y;\x)$ is piecewise differentiable by assumption 1.
\begin{proof}
Following assumption 1, we use the classical formula of transformed random variables  (\cite{ross2014introduction}, section 2.5.4). Given the standard normal probability density function (PDF) $\phi(\cdot)$, the PDF of the original data is given as $p_{\thet}(y|\x)=|g'_{\bm{\theta}}(y;\x)|\phi(g_{\thet}(y;\x))$. Since $g'_{\thet}(y;\x)$ must be positive, we obtain the PDF of the data:   $p_{\bm{\theta}}(y|\x)=\frac{g'_{\bm{\theta}}(y;\x)}{\sqrt{2\pi}}e^{-\frac{1}{2}g_{\bm{\theta}}^2(y;\x)}
$, whose negative log is given in lemma 1.
Convexity is guaranteed by noting that the quadratic and negative logarithm are convex function. Next, $g$ is affine in $\thet$, and consequently $g'$ is affine in $\thet$ too. Finally, convexity is preserved under affine transformations. 
\end{proof}

Given a training dataset $\{(\x_i,y_i)\}_{i=1,..,n}$, CNR is defined as the maximum likelihood estimate of $\thet$: 
\begin{eqnarray}
\label{CNR_ml}
\begin{array}{ll}
   \min_{\thet}  &  \sum_i g_{\thet}^2(y_i;\x_i)-2\log\(
g'_{\bm{\theta}}(y_i;\x_i)
\)\\
    {\rm{s.t.}} & \thet \in \Theta(\x_i) \; \quad i=1,\cdots,n 
\end{array}
\end{eqnarray}
Convexity ensures that this minimization can be efficiently solved using existing toolboxes. This property allows us to add convex penalties and/or constraints. We can use this to regularize the objective when the number of samples is insufficiently large compared to the dimension of ${\bm{\theta}}$.  


As discussed, CNR generalizes two well known cases:
\begin{lemma}[Generalization of LR]
If $g$ is jointly linear in $\x$ and $y$ then CNR reduces to standard linear regression of $y$ given $\x$:
\begin{eqnarray*}
y|\x\sim \mathcal{N}\left(\w_{\thet}^T \x,\sigma^2_{\thet}\right)
\end{eqnarray*}
\end{lemma}
\begin{proof}
This follows immediately by defining $g_{\thet}(y;\x)=\u_{\thet}^T\x+v_{\thet}y$, $\w_{\thet}=\frac{-\u_{\thet}}{v_{\thet}}$ and $\sigma^2_{\thet}=\frac{1}{v^2_{\thet}}$.

\end{proof}
\begin{lemma}[Generalization of GR]
If $g$ is affine in $y$ then CNR reduces to maximum likelihood estimation of a Gaussian distribution in its canonical form:
\begin{eqnarray*}
y\sim \mathcal{N}\left(w_{\thet}(\x),\sigma^2_{\thet}(\x)\right)
\end{eqnarray*}
\end{lemma}
\begin{proof}
This follows immediately by defining $g_{\thet}(y;\x)=u_{\thet}(\x)+v_{\thet}(\x)y$, $w_{\thet}(\x)=\frac{-u_{\thet}(\x)}{v_{\thet}(\x)}$ and $\sigma^2_{\thet}(\x)=\frac{1}{v^2_{\thet}(\x)}$.
\end{proof}

Thus, with proper constraints, LR and GR can both be implemented as special cases of CNR.

CNR involves a convex yet constrained optimization. The optimization is performed on a training dataset and the constraints are enforced on these samples. At test time, the resulting posterior distribution is only valid if the constraints are satisfied. When the number of training samples is sufficiently large, the constraints are typically satisfied on most test samples. Otherwise, we recommend to fall back to a simpler and valid estimator as LR in invalid test points. We emphasize that it is easy to identify these problematic samples as the constraints only depend on the features, and do not involve the unknown labels. 

\section{Piecewise Linear CNR}
In this section, we provide a concrete construction of CNR. 
We introduce a specific transformation class that satisfies assumptions 1-2 and offers a flexible tradeoff between complexity and expressiveness. It also gives rise to a closed form solution of the posterior mean, and can be efficiently computed using a scalable ADMM algorithm.

The class is based on a dictionary of piecewise linear, monotonic non-decreasing functions of $y$ with weights that are affine functions of the features. For a pair of observation $(\x,y)$ we therefore have:
\begin{equation}
    g_{\bm{\thet}}(y;\x)
    = \h^T(y)\u(\x), \quad 
    \u(\x)
    = \A\boldsymbol{\psi}(\x)+\b
\label{g_ab}
\end{equation}
Here, $\h\(y\)$ is the dictionary defined as:
\begin{equation}
\label{h_}
    \h^T\(y\)=[1,h_0(y),...h_{L+1}(y)],
\end{equation}
where each of the basis functions $h_i$ is monotonically non-decreasing in $y$. Moreover, $\boldsymbol{\psi}(\x)\in\mathbb{R}^k$ is a pre-defined feature mappings (which can also be the identity mapping, i.e equals to $\x$ itself), and
${\bm{\theta}}=\{\A\in\mathbb{R}^{(L+3)\times k},\b\in\mathbb{R}^{L+3}\}$ is the set of unknown parameters.

We construct our dictionary using simple step basis functions, which are differentiable except from one or two points. More specifically, let $I=\[p_0,...,p_L\]$ be evenly spaced points in the real line, such that
\begin{equation*}
\big{\{}\delta_{j+1}=p_{j+1}-p_{j}\big{\}}_{j=0,...,L-1}
\end{equation*}
are the distances between every two adjacent points. Note that the index of $\delta_j$ run from $1$ to $L$. We then define: 
\begin{equation*}
h_{j}(y)=
  \left\{
\begin{array}{ll}
     0 & y < p_{j-1}\\
      y- p_{j-1} & p_{j-1} \leq y < p_{j} \\
      \delta_{j}   & p_{j}  \leq y \\
\end{array} 
\right.
\quad \forall j\in\{1,...,L\}
\end{equation*}
Below and above $p_0$ and $p_L$ we define, respectively:
\begin{equation*}
    \begin{split}
    h_0(y)=
    \Big{\{}
    \begin{array}{ll}
     y-p_0 & y < p_0\\
     0   & p_{0}\leq y \\
\end{array}
,
\\
h_{L+1}(y)=
    \Big{\{}
    \begin{array}{ll}
     0 & y < p_L\\
     y-p_L   & p_{L}\leq y \\
\end{array}
.
\end{split}
\end{equation*}

Finally, to ensure a monotonic increasing $g$, we define the domain $\Theta_{\x}$:
\begin{equation*}
    \Theta_{\x} = 
    \big{\{}
    \(\A,\boldsymbol{\b}\):
    \forall j\geq 2, \:
    \[\A\boldsymbol{\psi}(\x)+\boldsymbol{\b}\]_{j}>0
    \big{\}}.
    \label{Thetx}
\end{equation*}
It is easy to see that this construction satisfies assumptions 1-2 and provides a convex CNR. 
Reminiscent of kernel density estimation \cite{wasserman2006all}, the dictionary divides the domain of $y_i$ into bins $L+2$ bins through the set $I$, each with a different transformation. With a sufficient number of bins, any monotonic increasing transformation can be modeled.

\textbf{The Posterior Mean}.
As discussed, with CNR we can characterize arbitrary conditional distributions. In some applications, however, we are less interested in the full shape of the posterior and simply require point-wise prediction of unknown labels. 
In particular, we are interested in computing the posterior mean, which is closely related to the Minimum Mean Squared Error (MMSE) \cite{kay1993fundamentals}. 
$\min_{\hat y}\mathbb{E}[(\hat{y}-y)^2]\quad \Rightarrow \quad \hat y_{\rm{mean}}=\mathbb{E}(y|\x)$.

It is possible to compute the posterior mean in closed form for the piecewise linear CNR, given by the next lemma
(see proof in supplementary material)
:

\begin{lemma}[Posterior mean]
\label{posteriorMean}
The conditional mean of $y$ given $\x$ in the piecewise linear CNR model is 
\begin{multline}
\label{CNRE}
 \hat y_{\rm{mean}}=
     -\frac{e^{-\frac{1}{2}\mu^2}}{\sqrt{2\pi}\alpha_0}
    +\(p_0-\frac{\mu}{\alpha_0}\) \Phi(\mu)+\\
    \sum\limits_{j=0}^{L-1}
    \Bigg{(}
    \frac{e^{-\frac{1}{2}\(\mu+\Delta_j\)^2}
    -e^{-\frac{1}{2}\(\mu+\Delta_{j+1}\)^2}}{\sqrt{2\pi}\alpha_{j+1}}
    \\
    +\(p_j-\frac{\mu+\Delta_j}{\alpha_{j+1}}\) \(\Phi(\mu+\Delta_{j+1})-\Phi(\mu+\Delta_{j})\)
    \Bigg{)}+
    \\
    \frac{e^{-\frac{1}{2}\(\mu+\Delta_L\)^2}}{\sqrt{2\pi}\alpha_{L+1}}
    +\(p_L-\frac{\mu+\Delta_L}{\alpha_{L+1}}\) \(1-\Phi(\mu+\Delta_{L})\)
    ,
\end{multline}
where $\Phi$ is the standard normal cumulative distribution function (CDF) and we define
\begin{align}
    [\mu,\alpha_0,...,\alpha_{L+1} ]^T= \u(\x)=\A
    \boldsymbol{\psi(\x)}
    + \boldsymbol{b} \nonumber\\
    \Delta_0 = 0
    \label{eq:lemma5}, \quad
    \Delta_j  = \delta_1 \alpha_1+...+\delta_j \alpha_j
\quad \forall j=1,...,L. \nonumber
\end{align}
\end{lemma}

\textbf{ADMM implementation}.
One of the main advantages of CNR models is their reliance on a tractable convex optimization problem. 
In fact,
our likelihood objective can be reparameterized 
and then minimized using the ADMM machinery \cite{boyd2011distributed}, which is used in large scale problems.
Our goal is to minimize the NLL, defined by  Equations \eqref{CNR_ml}-\eqref{g_ab}, over the parameters $\A,\b$ . Thus, in order to develop the ADMM, we first need to write the objective as a function of $\A,\b$ given as a vector. Stacking all entries of $\A$ into one vector, we define:
$\a = 
    \[A_{1,1},...,A_{1,k},...,
    A_{L+3,1},...,A_{L+3,k}
    \]$.
Together with $\b$ the vector of unknown parameters is $[\a,\b]$. It can be easily verified that the transformation \eqref{g_ab} and its derivative are: 
\begin{align*}
\nonumber
    g_{\thet}(y_i;\x_i)=
    \[ \h^T(y_i)\otimes 
    \boldsymbol{\psi}(\x_i)^T \quad 
    \h^T(y_i)
    \] 
    \begin{bsmallmatrix}
    \a \\ \b
    \end{bsmallmatrix}
    =\p_i^T
    \begin{bsmallmatrix}
    \a \\ \b
    \end{bsmallmatrix}\\
g'_{\thet}(y_i;\x_i)=
    \[ {\h'}^T(y_i)\otimes 
    \boldsymbol{\psi}(\x_i)^T \quad 
    {\h'}^T(y_i)
    \] 
    \begin{bsmallmatrix}
    \a \\ \b
    \end{bsmallmatrix}
    ={\p'}_i^T
    \begin{bsmallmatrix}
    \a \\ \b
    \end{bsmallmatrix},
\end{align*}
where
$\otimes$ is the Kronecker product,
and ${\h'}^T(y_i)$ is the derivative of ${\h}^T(y_i)$ over $y_i$. 
We can now define: 
\begin{equation}
    \P = \sum_{i=1}^n \p_i \p_i^T , \quad
    \P'= \begin{bsmallmatrix}
    \hdots & {\p'}^T_1 & \hdots \\ 
     &  \vdots  &  \\
     \hdots & {\p'}^T_n & \hdots
    \end{bsmallmatrix}
    \label{PdP}
\end{equation}
Altogether, the objective (\ref{CNR_ml}) boils down to:
\begin{equation*}
    \min_{\a,\b} \[\a^T,\b^T\] \P 
    \begin{bsmallmatrix}
    \a \\ \b
    \end{bsmallmatrix}
    -2\sum_{i=1}^n \text{log}\([\P' \begin{bsmallmatrix}
    \a \\ \b
    \end{bsmallmatrix}]_i\).
\end{equation*}
ADMM solves this minimization problem
by replacing each $i$-th scalar in the logarithm by $z_{i}$, using the constraint $z_{i}>0$ (to ensure a monotonically increasing $g$) through the augmented Lagrangian \cite{boyd2011distributed}.
It then alternatingly solves for $\[\a,\b\]$, $z_{i}$ and $y_{i}$. A pseudocode for the ADMM is given in 
Algorithm 1 with penalty parameter $\rho$. See
supplementary material for a detailed derivation of the ADMM. While LR solution is given in closed form, the ADMM includes additional $n$ calculations of $z_i$. Moreover, ADMM is iterative and it requires a large number of iterations (as ADMM usually converges slowly \cite{boyd2011distributed}).  

\begin{algorithm}[]
Input: $\rho>0$, $\{\x_i,y_i\}_{i=1}^n$, dictionary $\h(\cdot)$\\
Initialize $\z\in \mathbb{R}^n$\\
Compute $\P,\P'$ \eqref{PdP}, using $\{\x_i,y_i\}_{i=1}^n$ and $\h(\cdot)$  \\
\textbf{Repeat until convergence}
\begin{itemize}
\item $\w\leftarrow -\frac{1}{2}
\(\P+\frac{\rho}{2}{\P'}^T {\P'} 
\)^{-1}
    {\P'}^T 
    \(
    \y-\rho \z
    \)$
\item $z_{i}\leftarrow
\frac{\rho [{\P'} \w]_i+y_{i} 
    +
    \sqrt{\(\rho [{\P'} \w]_i+y_{i}\)^2+8\rho}
    }
    {2\rho} \quad\forall i=1,..,n$.
\item $\y \leftarrow \y+\rho\(
{\P'}\w-\z
\)$
\end{itemize}
\textbf{Return} $ \begin{bsmallmatrix} 
\a \\ \b \end{bsmallmatrix} \leftarrow \w$.
 \caption{\tt{ADMM}}
 \label{admm}
\end{algorithm}

\section{Experimental Evaluation}

\begin{figure*}[!htb]
\minipage{0.32\textwidth}
\caption*{LR}
  \includegraphics[width=\linewidth]{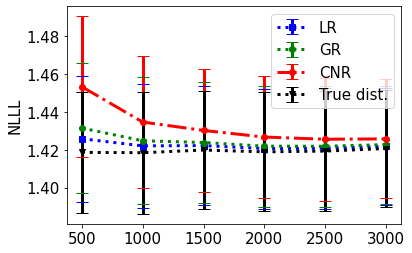}
\endminipage\hfill
\minipage{0.32\textwidth}
\caption*{CNR}
  \includegraphics[width=\linewidth]{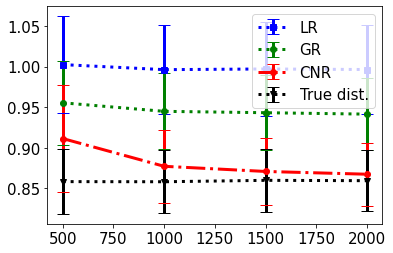}
\endminipage\hfill
\minipage{0.32\textwidth}%
\caption*{MR}
  \includegraphics[width=\linewidth]{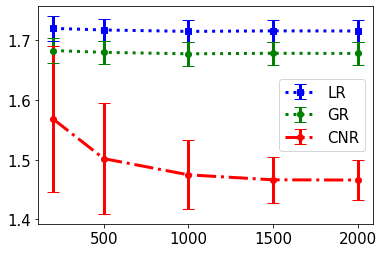}
\endminipage\hfill
\\
\minipage{0.32\textwidth}
  \includegraphics[width=\linewidth]{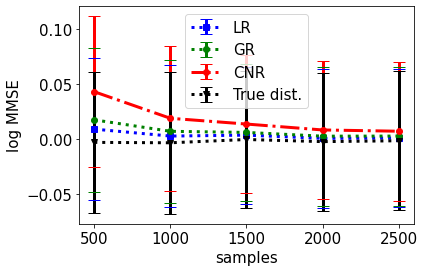}
\endminipage\hfill
\minipage{0.32\textwidth}
  \includegraphics[width=\linewidth]{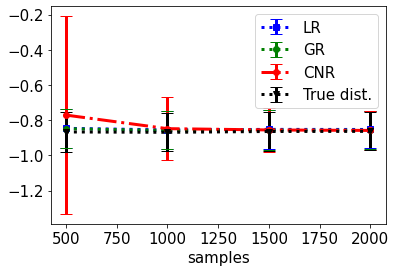}
\endminipage\hfill
\minipage{0.32\textwidth}%
  \includegraphics[width=\linewidth]{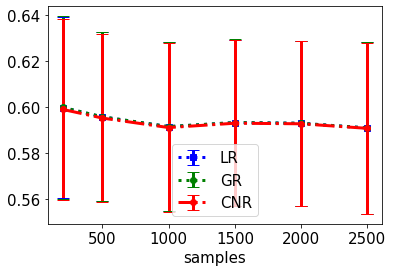}
\endminipage\hfill
\caption{NLL and MMSE  values (top and bottom rows, respectively) as a function of the number of training samples, given LR, CNR and MR synthetic data (left, middle and right columns, respectively).  
    }
    \label{minipage1}
\end{figure*}

\begin{figure}
    \centering
    \includegraphics[width=0.35\textwidth,height=3.1cm]{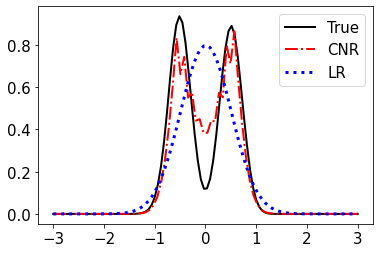}
    \caption{True MR and estimated CNR PDF's for one instance of $\x_i$.}
    \label{fig:MR}
\end{figure}

\begin{figure}
    \centering
    \includegraphics[width=0.35\textwidth,height=3.1cm]{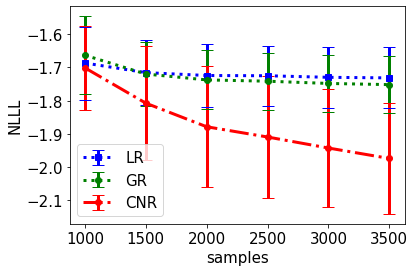}
    \caption{NLL reusults for household data.}
    \label{fig:household}
\end{figure}

We now demonstrate the efficacy of our approach in capturing
expressive conditional distributions. We compare our
CNR approach to the LR and GR baselines. Running CNR we use four bins, partitioned by the training data's $(0.3,0.5,0.7)$ quantiles. In case where synthetic data is CNR distributed we use true bins instead. In the synthetic simulations,
we also report a clairvoyant lower bound based on the true
unknown parameters. Since our main goal is to capture the
model and its posterior probabilities, we measure performance through the negative log likelihood (NLL) of an independent
test set. We also provide log MMSE results with respect to the prediction error of $y_i$,
where is the conditional expectation in Lemma 4, showing
asymptotically comparable results for all models. The reported
metrics are empirical means of 2,000 independent trials per
point, together with error bars of one standard deviation. For synthetic data simulations, the true parameters are drawn from random samples.

\textbf{Synthetic data.} We use data
generated from the LR and CNR regression models. The true
parameters are fixed constants and the features are standard
normal variables. Fig. \eqref{minipage1} provides NLL and log MMSE results as vs. the number of samples generated from LR and
CNR settings (left and middle columns). In terms of NLL, each
estimator performs best under its true model.
CNR requires
more samples to reach its competitors’ performance in their
specialized models. Nevertheless, CNR is significantly better
in terms of NLL when the true data is CNR distributed. In
terms of MMSE, LR typically exceeds in sufficiency. We also
consider a simple mixture of regressions (MR)  \cite{balakrishnan2017statistical}, where the $i$-th observation is $y_i=\pm \a^T \x_i+0.2\epsilon_i$, $\epsilon_i \sim \mathcal{N}(0,1)$, the +/- is randomly selected with probability of 1/2, leading to a bimodal PDF. The right column of Fig. 1 shows that
CNR beats its competitors in terms of NLL. To illustrate CNR’s
expressiveness, we plot the histogram of an estimated CNR for one instance of $\x_i$. Fig. \ref{fig:MR} shows that CNR captures MR bimodality.

\textbf{Real world data}. We  consider household electric power consumption\footnote{\url{https://archive.ics.uci.edu/ml/datasets/Individual+household+electric+power+consumption#}} time series, with the goal
of predicting future outcomes. We extract $n$ random sub-sequences of length $10$ and use the first $9$ variables to predict the last variable. At
each trial we assign 500 samples to the test set, and
the remaining samples to the training set. As previously,
we estimate the LR, GR and CNR parameters using the training set, and compute the NLL and MMSE of the test set. We preform
$2,000$ trials per $n$. As shown in Fig. \ref{fig:household}, the CNR results have better NLL values. In terms of MMSE, all models perform similarly. 

\section{Conclusion}
In this paper we introduced CNR, a novel nonparanormal approach for modeling flexible conditional posterior distributions. Despite being able to capture arbitrary distributions, CNR optimization objective is a convex one. We made the approach concrete using a dictionary of simple non-linear basis transformation and derived closed forms for the predictive mean in this setting. Finally, we demonstrated the merit of our approach on synthetic and real data.

One main issue left unresolved is the effect of training sample size on the domain of our CNR estimator, as we noticed in simulations the presence of invalid test set predictions. Thus, a possible future direction is a methodology development for specifying more general and fewer constraints that will define a larger domain.  
Another important issue is the choice of dictionary and bins that can affect model expressiveness and accuracy. 
All these are left for further investigation.

\clearpage
\newpage

\bibliography{ref}
\bibliographystyle{ieeetr}

\clearpage
\newpage

\section*{Supplemental: Convex Nonparanormal Regression}

Supplemaentary material includes the proof of lemma 4 and ADMM derivation.

\textbf{Proof of lemma  4}
We obtain the result by computing the expectation in each segment of the real line, partitioned by $I$.
The CNR density function $p_{\thet}(y|\x)$ in each segment is:
\begin{multline*}
\begin{cases}
 \frac{\alpha_0}{\sqrt{2\pi}}e^{-\frac{1}{2}(\mu+\alpha_0\(y-p_0\))^2} & y\leq p_0\\
  \frac{\alpha_{j+1}}{\sqrt{2\pi}}e^{-\frac{1}{2}\(\mu+
  \Delta_j\)
  +\alpha_{j+1}\(y-p_{j}\))^2} &  y\in(p_j, p_{j+1}],
  0\leq j\leq L-1\\
  \frac{\alpha_{L+1}}{\sqrt{2\pi}}e^{-\frac{1}{2}(\mu+
  \Delta_L
  +\alpha_{L+1}\(y-p_{L}\))^2} &  y>p_L
 \end{cases}
\end{multline*}
Considering the integral over each segment, it can be verified that: 
\begin{equation*}
    \int_{-\infty}^{p_0}
    yp_{\thet}(y|\x)dy=
    -\frac{e^{-\frac{1}{2}\mu^2}}{\sqrt{2\pi}\alpha_0} 
    +\(p_0-\frac{\mu}{\alpha_0}\) \Phi(\mu),
\end{equation*}
\begin{multline*}
    \int^{\infty}_{p_L}
    y
    p_{\thet}(y|\x)
    dy=
    \\
    \frac{e^{-\frac{1}{2}\(\mu+\Delta_L\)^2}}{\sqrt{2\pi}\alpha_{L+1}}
    +\(p_L-\frac{\mu+\Delta_L}{\alpha_{L+1}}\) \(1-\Phi(\mu+\Delta_{L})\),
\end{multline*}

and for the inner segments, $j=1,...,L-1$:
\begin{multline*}
\int^{p_{j+1}}_{p_j}
    yp_{\thet}(y|\x)
    dy=
    \frac{e^{-\frac{1}{2}\(\mu+\Delta_j\)^2}
    -e^{-\frac{1}{2}\(\mu+\Delta_{j+1}\)^2}}{\sqrt{2\pi}\alpha_{j+1}}
    \\
    +\(p_j-\frac{\mu+\Delta_j}{\alpha_{j+1}}\) \(\Phi(\mu+\Delta_{j+1})-\Phi(\mu+\Delta_{j})\).
\end{multline*}
The required result is obtained by summing up all these terms.

\textbf{ADMM derivation}
We first derive a matrix form expression for the likelihood $\sum_i g_{\thet}^2(y_i;\x_i)-2\log\(
g'_{\bm{\theta}}(y_i;\x_i)
\)$.  Define:
 \begin{equation*}
     \H = 
     \begin{bsmallmatrix}
     1 & h_{0}(y_1) & \hdots &  h_{L+1}(y_1) \\
     \vdots & \vdots & \vdots & \vdots \\ 
     1 & h_{0}(y_n) & \hdots &  h_{L+1}(y_n) \\
     \end{bsmallmatrix},
     \H' = 
     \begin{bsmallmatrix}
     0 & h^{'}_{0}(y_1) & \hdots &  h^{'}_{L+1}(y_1) \\
     \vdots & \vdots & \vdots & \vdots \\ 
     0 & h^{'}_{0}(y_n) & \hdots &  h^{'}_{L+1}(y_n) \\
     \end{bsmallmatrix}.
 \end{equation*}
Denote by $\h_i$ and $\h^{'}_i$ the $i$'th row of $\H$ and $\H'$, respectively.
 Then $g_{\thet}(y_i;\x_i)=\h_i^T \(\A\x_i +\boldsymbol{b}\)$ and 
 $g'_{\thet}(y_i;\x_i)={\h'_i}^T \(\A\x_i +\boldsymbol{b}\)$. 
 Thus, the likelihood 
 can be written as:
 \begin{equation*}
     \sum_i
     \( 
     \h_i^T \(
     \A\x_i +\boldsymbol{b}
     \)
     \)^2
     -2\text{log} 
     \(
     {\h'_i}^T \(\A\x_i +\boldsymbol{b}\)
     \).
 \end{equation*}

Denote the row-wise vectorized form of $\A$ by $\a$, such that:
\begin{multline*}
    \a=
    \[\A_{1:},..., \A_{(L+3):}\]
    = 
    \[A_{1,1},...,A_{1,k},...,
    A_{L+3,1},...,A_{L+3,k}
    \],
\end{multline*}
where $\A_{i:}$ is the $i$'th row of $\A$. 
We now convert the terms $\h^T_i(\A\x_i+\boldsymbol{b})$ and ${\h'}^T_i(\A\x_i+\boldsymbol{b})$ into a vector forms. It holds that:
$\h^T_i(\A\x_i+\boldsymbol{b})=
\sum\limits_{\ell=1}^{L+3}
    \[\h_{i}\]_{\ell} (\x_i^T \A_{\ell:}+b_{\ell} )$,
which is equal to
$ \(\h_i^T \otimes \x_i^T\) \a 
    + \h_i^T \boldsymbol{b}
    = 
    \[\h_i^T \otimes \x_i^T 
    \quad 
    \h_i^T
    \] \w$,
where 
$ \w = \begin{bsmallmatrix}
\a \\
\boldsymbol{b}
\end{bsmallmatrix}$.
Similarly,
\begin{equation*}
    {\h'}^T_i(\A\x_i+\boldsymbol{b})
    =
    \[{\h'}_i^T \otimes \x_i^T 
    \quad 
    {\h^{'}_i}^T
    \] \w
    . 
\end{equation*}
Using the notation 
$ \p^T_i=\[\h_i^T \otimes \x_i^T 
    \quad 
    \h_i^T
    \]$
    and 
$\q^T_i=\[{\h'_i}^T \otimes \x_i^T 
    \quad 
    {\h^{'}_i}^T
    \]$,
the likelihood can now be written as:
\begin{equation*}
    \w^T \(\sum_i
    \p_i \p^T_i
    \) \w
    -2\sum_i \text{log}\(\q^T_i \w\),
\end{equation*}
which is also equivalent to:
\begin{equation*}
    \w^T \(\sum_i
    \p_i \p^T_i
    \) \w
    -2\sum_i \text{log}\(z_i\), \quad s.t. \quad \z=\Q\w,
\end{equation*}
where 
$ \Q=\begin{bsmallmatrix}
\q^T_1 \\ 
\vdots \\
\q^T_n
\end{bsmallmatrix}.$
The ADMM objective now becomes:
\begin{multline*}
    \w^T \(\sum_i
    \p_i \p^T_i
    \) \w 
    -2\sum_i \text{log}\(z_i\) 
    + \y^T (\Q\w-\z) + \frac{\rho}{2}
    \|\Q\w-\z\|_2^2.
\end{multline*}
We now derive the alternating minimization steps. First, with respect to $\w$ the objective is quadratic. Ignoring all terms free of $\w$, the objective is:
\begin{multline*}
    \w^T \(\sum_i
    \p_i \p^T_i
    \) \w 
    +\y^T \Q \w + 
    \frac{\rho}{2} \w^T \Q^T \Q\w 
    -\rho \z^T \Q\w=\\
    \w^T \(\sum_i
    \p_i \p^T_i
    +\frac{\rho}{2}\Q^T \Q 
    \) \w
    +
    \(
    \y^T\Q -\rho \z^T\Q
    \)\w,
\end{multline*}
whose minimizer is given by:
\begin{equation*}
    \w_{min}= -\frac{1}{2}
    \(\sum_i
    \p_i \p^T_i
    +\frac{\rho}{2}\Q^T \Q 
    \)^{-1}
    \Q^T 
    \(
    \y-\rho \z
    \).
\end{equation*}
Next, ignoring all terms free of $\w$ the objective becomes:
\begin{equation*}
    \mathcal{L}=
    -2\sum_i \text{log}\(z_i\) 
    -\y^T\z + \frac{\rho}{2}\z^T\z
    -\rho \w^T \Q^T \z.
\end{equation*}
We then have:
\begin{multline*}
    \frac{\partial \mathcal{L}}{\partial z_{\ell}} = 
    -\frac{2}{z_{\ell}} -y_{\ell} 
    +\rho z_{\ell} -\rho \q_{\ell}^T \w=0 \implies \\
      \rho z_{\ell}^2 
      -z_{\ell} \(\rho  \q_{\ell}^T\w
      +y_{\ell}\)
    -2
    =0.
\end{multline*}
There are two possible solutions to this equation:
\begin{equation*}
    \frac{\rho \q_{\ell}^T\w+y_{\ell} 
    \pm 
    \sqrt{\(\rho \q_{\ell}^T\w+y_{\ell}\)^2+8\rho}
    }
    {2\rho}.
\end{equation*}
Since $\rho>0$, the only positive solution is
\begin{equation*}
    z_{\ell}^*=
    \frac{\rho \q_{\ell}^T\w+y_{\ell} 
    +
    \sqrt{\(\rho \q_{\ell}^T\w+y_{\ell}\)^2+8\rho}
    }
    {2\rho}.
\end{equation*}

Together with the dual update, each ADMM iteration consists of three step as described in  Algorithm 1 in the main text.

\end{document}